\documentclass[letterpaper, 10 pt, conference]{ieeeconf}  

\IEEEoverridecommandlockouts                              

\usepackage{cite}
\usepackage{amsmath,amssymb,amsfonts}
\usepackage{algorithmic}
\usepackage{graphicx}
\usepackage{textcomp}
\usepackage{xcolor}
\usepackage{cite}
\usepackage{amsmath,amssymb,amsfonts}
\usepackage{algorithmic}
\usepackage{graphicx}
\usepackage{textcomp}
\usepackage{xcolor}
\usepackage{booktabs}
\usepackage{multirow}
\usepackage{float}
\usepackage{stfloats}
\usepackage{graphicx}
\usepackage{placeins}
\usepackage{hyperref}
\usepackage{algorithm}
\usepackage{comment}

\newtheorem{definition}{Definition}
\newtheorem{theorem}{Theorem}

\title{Kernel-Based Enhanced Oversampling Method for Imbalanced Classification}

\author{Wenjie LI\(^{1,2}\), Sibo Zhu\(^{1,2}\), Zhijian Li\(^{1,2}\), and Hanlin Wang\(^{1,2*}\)
\thanks{*Corresponding author: Hanlin Wang (r130026135@mail.uic.edu.cn)}
\thanks{\(^{1}\)Faculty of Science and Technology, BNU-HKBU United International College, Zhuhai, China}%
\thanks{\(^{2}\)Guangdong Provincial Key Laboratory of Interdisciplinary Research and Application for Data Science, BNU-HKBU United International College, Zhuhai, China}%
\thanks{Emails: \{r130026076, r130026221\}@mail.uic.edu.cn, zhijianli@uic.edu.cn, r130026135@mail.uic.edu.cn}%
}

\begin{document}

\maketitle
\thispagestyle{empty}
\pagestyle{empty}

\begin{abstract}
This paper introduces a novel oversampling technique designed to improve classification performance on imbalanced datasets. The proposed method enhances the traditional SMOTE algorithm by incorporating convex combination and kernel-based weighting to generate synthetic samples that better represent the minority class. Through experiments on multiple real-world datasets, we demonstrate that the new technique outperforms existing methods in terms of F1-score, G-mean, and AUC, providing a robust solution for handling imbalanced datasets in classification tasks.
\end{abstract}

\section{Introduction}

Imbalanced datasets are a pervasive issue in the domain of classification, where the distribution of classes is skewed, with one class (often referred to as the minority class) being significantly underrepresented compared to the other (the majority class). The imbalance issue is especially problematic in classification tasks, as traditional machine learning algorithms are generally designed to maximize overall accuracy, leading them to favor the majority class. Consequently, it results in a bias where the model performs well on the majority class but poorly on the minority class, which is often the class of greater interest \cite{rawat2022reviewmethodshandlingclassimbalanced}.

The implications of this issue are evident in a wide range of real-world applications. For example, in fraud detection, fraudulent transactions are rare compared to legitimate ones, yet accurately identifying these minority instances is crucial \cite{ali2022financial}. Similarly, in medical diagnosis, diseases such as cancer are far less common than healthy outcomes, but the cost of missing a diagnosis is extremely high \cite{jiang2023deep}. In many scenarios, the presence of an imbalanced dataset not only diminishes the effectiveness of the classification model but also raises concerns about its reliability and fairness.

Addressing the challenges posed by imbalanced datasets is therefore critical to improving the accuracy and robustness of classification models, especially in applications where the minority class holds significant importance. Several strategies have been developed to tackle the issues associated with imbalanced datasets, and the most popular methods include sampling methods, cost-sensitive learning, and ensemble methods.

The two primary techniques to balance the class distribution are \textit{oversampling} and \textit{undersampling}. Oversampling methods like Synthetic Minority Over-sampling Technique (SMOTE) create synthetic examples for the minority class to balance the dataset. However, oversampling can lead to overfitting, especially if synthetic examples are too similar to existing minority instances \cite{Chawla_2002}. Undersampling reduces the number of majority class instances to balance the dataset, which can be effective but risks losing valuable information from the majority class \cite{LIN201717}.

The Cost-Sensitive Learning adjusts the learning algorithm to pay more attention to the minority class by assigning higher misclassification costs to minority class errors. This method directly integrates the imbalanced nature of the dataset into the model's learning process, making the model more sensitive to minority class instances\cite{9064578}.

Each of these methods has its strengths and limitations, and the choice of method often depends on the specific characteristics of the dataset and the application. For instance, in scenarios where the minority class is extremely rare and critical, such as in medical diagnoses or fraud detection, a combination of techniques might be necessary to achieve reliable performance.

This paper introduces a novel enhancement to the SMOTE algorithm, called KWSMOTE, to improve classification performance on imbalanced datasets by generating more realistic synthetic samples through a kernel-based weighting system. 

\section{Related Work}
Numerous approaches have been developed to address imbalanced learning. Over the past two decades, various surveys and literature reviews have organized and examined the distinct features and characteristics of these methods \cite{1japkowicz2002class,2chawla2004special,3su2011outlier,4he2009learning,5chawla2010data,6branco2016survey, IPLOF}. This is primarily due to the continuous emergence of new real-world applications, where data imbalance is inherently present \cite{7haixiang2017learning}. Recently, many imbalanced learning techniques have been integrated into Python libraries like \textit{imblearn} \cite{8lemaavztre2017imbalanced}, focusing on either modifying training data or adjusting algorithms to address class misclassification costs.

One notable approach to addressing class imbalance is the Synthetic Minority Oversampling Technique (SMOTE), proposed by \cite{Chawla_2002}. Essentially, SMOTE generates new samples by performing random linear interpolation between minority class samples and their nearest neighbors. Creating a specified number of synthetic minority samples reduces the data imbalance ratio, thereby enhancing the classification performance on imbalanced datasets. The detailed steps of SMOTE are outlined as follows:

\begin{enumerate}
    \item For each minority class sample $x_i$ with $i = 1, 2, \ldots, n$, compute the Euclidean distances to all other minority class samples to determine the $k$ nearest neighbors.
    \item Based on the desired oversampling rate $N$, randomly select $m$ neighbors from the $k$ nearest neighbors of each sample $x_i$. Denote these neighbors as $x_{ij}$ with $j = 1, 2, \ldots, m$.
    \item For each selected neighbor sample $x_{ij}$, create a new synthetic sample $p_{ij}$ using the formula:
    \begin{equation}
        p_{ij} = x_i + \text{rand}(0,1) \times (x_{ij} - x_i),
    \label{eq:smote_formula}
    \end{equation}
    where $\text{rand}(0,1)$ is a random value generated from a uniform distribution between 0 and 1. Continue this process until the dataset achieves the desired balance ratio.
\end{enumerate}

Many researchers have highlighted enhancements to the original SMOTE algorithm, focusing on different methods for generating synthetic samples. These modifications aim to better address the challenges posed by unbalanced datasets, making the improved versions more effective and suitable for various data imbalance situations.

One limitations of the classic SMOTE algorithm is that it can not properly handle boundary samples. SMOTE generates synthetic samples by linear interpolation between minority samples and their neighbors resulting in marginalization at the edges of the data distribution \cite{ImproveSMOTEwang2021research}. Data marginalization can blur the separation between classes, making classification more challenging. An improved SMOTE algorithm based on normal distribution was proposed to reduce the generation of edge samples by concentrating synthetic data near the minority class center. The new synthetic samples $p_i$ are generated using the following formula:

\begin{equation}
    p_i = x'_i + \text{rnorm}(\mu=1, \sigma) \cdot (x'_{\text{center}} - x'_i)
\end{equation}

\noindent where $\text{rnorm}(\mu=1, \sigma)$ is a random number drawn from a normal distribution with mean $\mu = 1$ and adjustable standard deviation $\sigma$. Compared to the uniform distribution in SMOTE, the normal weight ensures that the generated samples are more likely to cluster near the minority class center, thereby maintaining the internal structure of the minority class and preventing boundary samples from becoming overly marginalized.

Another drawback of SMOTE is that the generated synthetic samples are confined to the line segment between seed samples, which fails to represent the true distribution of the data. The SNOCC (Sigma Nearest Oversampling based on Convex Combination) \cite{SNOCCzheng2015oversampling} method was proposed to address this limitation. This approach extends the concept of linear combinations, allowing the generated synthetic samples to be distributed more broadly within the sample space, not only on the line segment. By expanding to multiple seed samples, SNOCC enables new samples to be distributed within the convex hull formed by the original samples. However, \cite{SNOCCzheng2015oversampling} doesn't provide any theoretical support and doesn't detail how to choose the seed number.

In summary, SMOTE and its variations have been widely used to address imbalanced datasets, but they still face limitations such as generating samples along line segments between minority instances and issues with boundary marginalization, which can blur class separability. Our approach aims to improve upon these drawbacks of SMOTE by developing a more adaptive oversampling technique that better replicates the true distribution of minority classes, preserves class boundaries, and enhances classification performance on imbalanced data.

\section{Methodology}

Before we formally elaborate on our methodology, several definitions are presented below.
\begin{definition}\label{def1}
    A \textit{convex set} is a set $\Omega$ in a vector space such that, for any two points $x_1, x_2 \in \Omega$, and any $w \in [0,1]$, the point $w x_1 + (1 - w) x_2$ also belongs to $\Omega$. 
\end{definition}

\begin{definition}\label{def2}
   A \textit{convex combination} of points $\{ x_0, x_1, \dots, x_k \}$ is any point $x$ that can be expressed as:
\begin{equation*}\label{convex_combination_definition}
    x = \sum_{j=0}^{k} w_j x_j,
\end{equation*}
where each $w_j \geq 0$ and $\sum_{j=0}^{k} w_j = 1$. 
\end{definition}

In our kernel-weighted SMOTE (KWSMOTE), we propose the following method to generate new samples $x_{ik}$,
\begin{align}\label{kw1}
    x_{ik} &= x_{i} + \frac{1}{D_{ik}} \sum_{j=0}^{k} w_{ij} (x_i^{(j)} - x_i) \nonumber \\ 
            & = \left(1-\frac{1}{D_{ik}} \sum_{j=0}^{k} w_{ij}\right)x_i + \frac{1}{D_{ik}} \sum_{j=0}^{k} w_{ij} x_i^{(j)},
\end{align}
where $x_{i}$ is the selected point, $x_i^{(j)}$ is the $j^{th}$ nearest neighbor point of $x_{i}$ and $x_i^{(0)} = x_i$, $w_{ij}$ is the weight and $D_{ik}$ is a coefficient. Note that if we let the coefficient $D_{ik} = \sum_{j=0}^{k} w_{ij}$, following definition \ref{def2}, equation (\ref{kw1}) can be written as a convex combination,
\begin{equation}\label{cb}
    x_{ik} = \nu_{i0} x_i^{(0)}+\nu_{i1} x_i^{(1)}+\dots+\nu_{ik} x_i^{(k)},
\end{equation}
where $\nu_{ij} = w_{ij}/D_{ik} > 0 $ and $\sum_{j=0}^{k} \nu_{ij} = 1$.
Consider a set of points within the $k^{th}$ nearest neighbor point of $x_{i}$, we can show that this is a convex set.

\begin{theorem}\label{thm1}
    The set 
    \[
    \Omega_{ik} = \{ x \mid \| x - x_i \| \leq L_{ik} \}
    \]
    is a convex set, where $L_{ik} = \| x_i^{(k)} - x_i \|$.
\end{theorem}

\begin{proof}
    Let $x, y$ be any two elements in $\Omega_{ik}$ and $w \in [0, 1]$. Define $z = w x + (1 - w) y$. We need to show that $z \in \Omega_{ik}$.
    Using the triangle inequality and properties of norms, we have:
    \begin{align*}
    \| z - x_i \| &= \left\| w (x - x_i) + (1 - w)(y - x_i) \right\| \\
    &\leq w \| x - x_i \| + (1 - w) \| y - x_i \| \\
    &\leq w L_{ik} + (1 - w) L_{ik} \\
    &= \left( w + 1 - w \right) L_{ik} \\
    &= L_{ik}.
    \end{align*}
    Therefore, $z \in \Omega_{ik}$, and so $\Omega_{ik}$ is convex.
\end{proof}

\begin{theorem}\label{thm2}
    If $\Omega$ is a convex set, then every convex combination of points in $\Omega$ is also in $\Omega$.
\end{theorem}

\begin{proof}

    \textbf{Base Case ($n = 1$):}

    For $n = 1$, the convex combination reduces to a single point $x = x_1$, where the weight is $w_1 = 1$. Since $x_1 \in \Omega$, it follows that $x \in \Omega$.

    \textbf{Inductive Step:}

    Assume that any convex combination of $k$ points in $\Omega$ is also in $\Omega$.
    For $k + 1$ points $x_1, x_2, \dots, x_{k+1} \in \Omega$ with weights $w_1, w_2, \dots, w_{k+1} \geq 0$ such that $\sum_{i=1}^{k+1} w_i = 1$, define $s = \sum_{i=1}^{k} w_i$.
    
    If $s = 0$, then $w_{k+1} = 1$, and so $x = x_{k+1} \in \Omega$.

    If $s > 0$, let
    \[
    y = \frac{1}{s} \sum_{i=1}^{k} w_i x_i.
    \]
    Since $\sum_{i=1}^{k} w_i/s = 1$ and each $w_i/s \geq 0$, $y$ is a convex combination of $k$ points in $\Omega$. By the induction hypothesis, $y \in \Omega$.

    Now, express $x$ as a convex combination of $y$ and $x_{k+1}$:
    \[
    x = s y + w_{k+1} x_{k+1}.
    \]
    Given that $s + w_{k+1} = 1$ and $s$, $w_{k+1} \geq 0$, and $y, x_{k+1} \in \Omega$, the convexity of $\Omega$ implies that $x \in \Omega$.

    Therefore, by induction, any convex combination of points in $\Omega$ is also in $\Omega$.
\end{proof}

Theorem \ref{thm2} ensures that the new synthetic samples lie within the convex hull of the data, instead of being on the line segment of two data points.

The next issue we need to consider is the selection of parameter $k$ and the weight $w_{ij}$. It is natural that if we select a large $k$, more data points will be considered. However, if the selected point or its neighbors contain noise or are distant, the newly generated points will also be affected. Therefore, we propose using the Gaussian kernel function as the weight
\[ 
K(x, x') = \exp \left( - \frac{\| x - x' \|^2}{2 \sigma^2} \right).
\]
The kernel weight value depends on the similarity between two points. A higher similarity results in a larger weight for a nearby point, giving it a greater influence, while a more distant point will have a smaller weight, reducing its impact. Noise and outliers typically result in negligible weights as the Gaussian kernel decays toward zero rapidly when the two input vectors are far apart. As distant points and even outliers have low effects, the kernel method avoids new points from being generated near the boundaries between classes, which helps to reduce the likelihood of blurring the boundary between minority and majority classes, as presented in Fig. \ref{fig:generate_image}. 

\begin{figure}[!ht]
    \centering
    \includegraphics[width=0.5\textwidth]{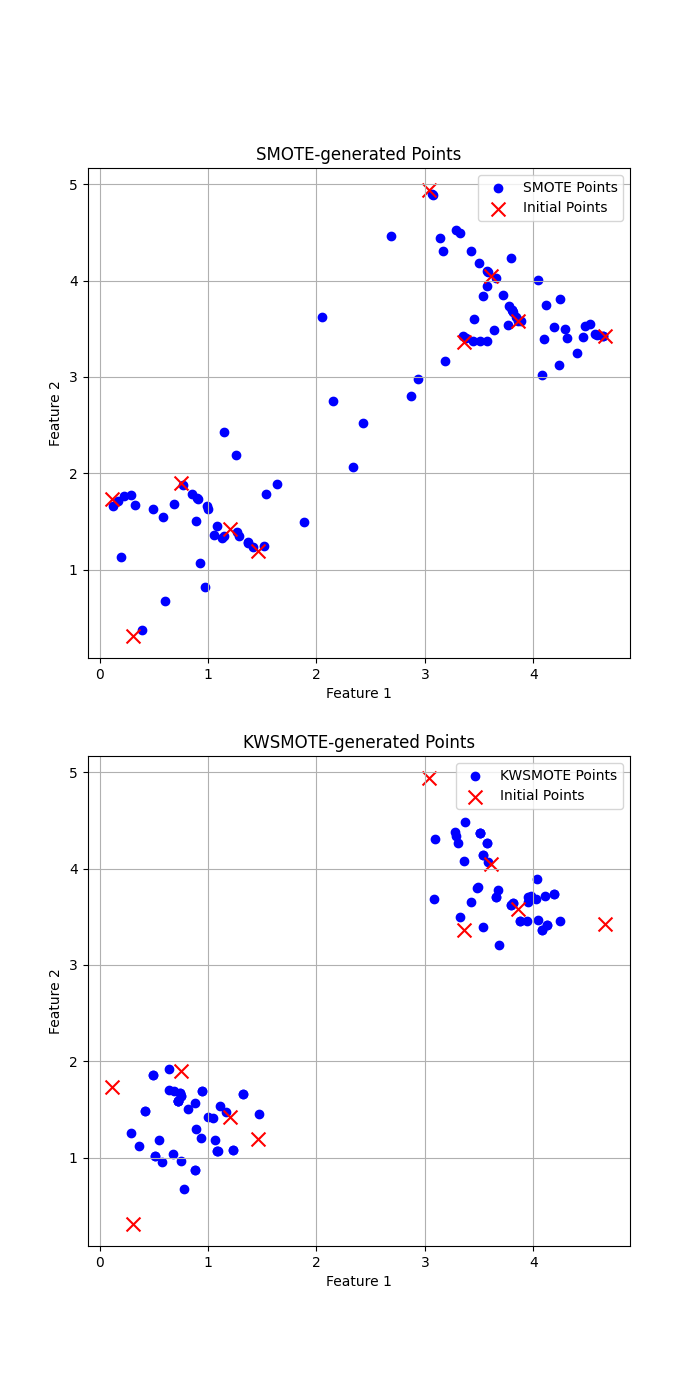}
    \caption{\small The generated points of SMOTE and KWSMOTE}
    \label{fig:generate_image}
\end{figure}

Applying the Gaussian kernel, the convex combination (\ref{cb}) can be written as:
\begin{equation} 
x_{ik} = \sum_{j=0}^{k} \frac{K(x_i, x_i^{(j)})}{D_{ik}} x_i^{(j)}.
\end{equation}

Specifically, the weight of the selected point $x_i = x_i^{(0)}$ in Equation~(\ref{kw1}) is
\begin{align} 
w_{i0} = K(x_i, x_i) = \exp \left( - \frac{\| x - x_i \|^2}{2 \sigma^2} \right) = 1,
\end{align}
which is the largest possible value. Therefore, the selected point always has the highest weight, making the newly generated points' position closer to the selected point, preserving local features, which can also be obtained by Fig. \ref{fig:generate_image}. 
It is also worth noting that if the kernel weight of the second nearest point $x_i^{(2)}$ is small and below a certain threshold, it indicates that both the selected point and its associated neighbor points are relatively distant, and the new point can be disregarded. This prevents the generation of a point in near places of outliers, avoiding interference with the classifier.

With the noise resistance in place, the nearest neighbor number \( k \) can be adjusted to a larger value without worrying about the impact of outliers. Moreover, the parameter \( k \) and the bandwidth \( \sigma \) of the Gaussian kernel can be finely tuned. By conducting hyperparameter optimization, we can obtain a set of parameters better suited to the dataset, allowing for the generation of more realistic points and improving the classifier's performance.

The pseudocode of the algorithm is given in Algorithm 1.

\begin{algorithm}[H]

\caption{KWSMOTE Algorithm}

\textbf{Input:} Dataset $(X, y)$, Number of neighbors $k$, Convex points number $c$, Threshold $\tau$, Kernel width $\sigma$ (optional)

\textbf{Output:} Resampled dataset $(X_{resampled}, y_{resampled})$

\begin{algorithmic}[1]

\STATE Identify minority class $y_{min}$ and majority class $y_{maj}$
\STATE Compute the number of samples to generate: $n_{samples\_to\_generate} = |y_{maj}| - |y_{min}|$

\STATE Extract minority class samples $X_{min} \leftarrow X[y = y_{min}]$

\STATE Compute $k$ nearest neighbors for each $x \in X_{min}$
\IF{$\sigma$ is None}
    \STATE Compute $\sigma = \sqrt{\frac{\text{Var}(X_{min}) \times n_{features}}{2}}$ (default setting in scikit-learn's SVM)
\ENDIF

\STATE Initialize empty list $new\_samples$

\WHILE{length(new\_samples) $<$ $n\_samples\_to\_generate$}

    \STATE Randomly select a minority sample $x_{i} \in X_{min}$
    \STATE Find its $k$ nearest neighbors $x_j$
    \STATE Randomly select $c$ neighbors from $x_j$, denoted as $x_j^{(j)}$
    \STATE Compute Gaussian kernel weights for each neighbor in $x_j^{(j)}$: 
    \[
    w_j = \exp\left( \frac{-\|x_{i} - x_j^{(j)}\|^2}{2\sigma^2} \right)
    \]
    \IF{$\max(w_j) < \tau$}
        \STATE Skip this sample
    \ENDIF
    \STATE Generate new sample:
    \[
    x_{ik} = \sum_{j=0}^{c} \frac{w_j}{D_{ik}} x_i^{(j)}
    \]
    \STATE Append $x_{ik}$ to $new\_samples$
\ENDWHILE

\STATE Combine original samples $X$ with $new\_samples$ to form $X_{resampled}$
\STATE Assign new labels to $y_{resampled}$
\STATE \textbf{Return:} $(X_{resampled}, y_{resampled})$

\end{algorithmic}
\end{algorithm}

\section{Evaluation}
\subsection{Experiment Environment}
Our experiment is done with dual Intel(R) Xeon(TM) E5-2699 V4 processors clocked at 2.20 GHz and having 128.00 GB DDR4 RAM on a home server with Ubuntu 20.04 operating system loaded. The implementation was done through Python 3.8.

The experiment was conducted by considering four datasets available online: Blood Transfusion \cite{blood_transfusion_service_center_176}, Haberman \cite{haberman's_survival_43}, Breast Cancer Wisconsin (Diagnostic) \cite{breast_cancer_wisconsin_(diagnostic)_17}, and Diabetes \cite{smith1988using}, each serving distinct medical classification purposes. The Haberman dataset focuses on the survival of breast cancer patients who underwent surgery, based on cases studied at the University of Chicago between 1958 and 1970.

The Breast Cancer Wisconsin (Diagnostic) dataset features attributes derived from images of fine needle aspirates (FNA) of breast masses, describing the characteristics of cell nuclei. The Blood Transfusion dataset, collected from the Blood Transfusion Service Center in Hsin-Chu City, Taiwan, China, includes donor data to model donation behavior using variables like Recency, Frequency, and Time. The Diabetes dataset, from the National Institute of Diabetes and Digestive and Kidney Diseases, predicts diabetes onset using diagnostic measurements. A detailed description of these datasets is provided in Table \ref{tab:datasets}.

\begin{table*}[!t]
\centering
\caption{Summary of the Datasets (\# and \% indicate quantity and rate respectively)}
\begin{tabular*}{\textwidth}{@{\extracolsep{\fill}} l c c c c c}
\toprule
\textbf{Dataset} & \textbf{\# Samples} & \textbf{\# Features} & \textbf{\# Minority} & \textbf{\# Majority} & \textbf{\% Imbalance} \\
\midrule
Blood Transfusion & 748 & 4  & 178 & 570 & 3.2  \\
Haberman          & 306 & 3  & 81  & 225 & 2.78 \\
Breast Cancer Wisconsin & 569 & 30 & 212 & 357 & 1.68 \\
Diabetes          & 768 & 8  & 268 & 500 & 1.87 \\
\bottomrule
\end{tabular*}
\label{tab:datasets}
\end{table*}

\begin{table*}[ht]
\centering
\caption{Evaluation Metrics for Different Datasets and Classifiers}
\begin{tabular*}{\textwidth}{@{\extracolsep{\fill}} l l l c c c c c c c c}
\toprule
\multirow{2}{*}{Dataset} & \multirow{2}{*}{Methods} & \multicolumn{3}{c}{RandomForest} & & \multicolumn{3}{c}{SVM} \\
\cmidrule(lr){3-5} \cmidrule(lr){7-9}
 &  & F1-score & G-mean & AUC & & F1-score & G-mean & AUC \\
\midrule
Blood & raw & 0.5937 & 0.6343 & 0.6778 & & 0.4406 & 0.5244 & 0.6241 \\
 & SMOTE & 0.5721 & 0.6009 & 0.6312 & & 0.5864 & 0.6223 & 0.6604 \\
 & BorderlineSMOTE & 0.5972 & 0.6190 & 0.6421 & & 0.5717 & 0.6245 & \textbf{0.6813} \\
 & SVMSMOTE & 0.6000 & 0.6222 & 0.6451 & & 0.6000 & 0.6155 & 0.6734 \\
 & SNOCC & 0.5893 & 0.6201 & 0.6650 & & \textbf{0.5900} & 0.6250 & 0.6702 \\
 & \textbf{KWSMOTE} & \textbf{0.6180} & \textbf{0.6343} & \textbf{0.6778} & & 0.5864 & \textbf{0.6285} & 0.6737 \\
\midrule
Haberman & raw & 0.5933 & 0.6318 & 0.6728 & & 0.4177 & 0.5403 & 0.6987 \\
 & SMOTE & 0.6150 & 0.6240 & 0.6332 & & 0.6338 & 0.6440 & 0.6544 \\
 & SNOCC & 0.6210 & 0.6349 & 0.6505 & & 0.6382 & 0.6517 & 0.6753 \\
 & BorderlineSMOTE & 0.5973 & 0.6229 & 0.6653 & & 0.5810 & 0.6062 & 0.6331 \\
 & SVMSMOTE & 0.5896 & 0.6194 & 0.6503 & & 0.6235 & 0.6522 & 0.6823 \\
 & \textbf{KWSMOTE} & \textbf{0.6336} & \textbf{0.6642} & \textbf{0.6964} & & \textbf{0.6428} & \textbf{0.6563} & \textbf{0.6987} \\
\midrule
Breast Cancer & raw & 0.9683 & 0.9824 & 0.9832 & & 0.9279 & 0.9600 & 0.9931 \\
 & SMOTE & 0.9687 & 0.9834 & 0.9984 & & 0.9424 & \textbf{0.9671} & 0.9924 \\
 & SNOCC & 0.9705 & 0.9842 & 0.9980 & & 0.9400 & 0.9660 & 0.9928 \\
 & BorderlineSMOTE & 0.9653 & 0.9820 & 0.9937 & & 0.9093 & 0.9508 & 0.9923 \\
 & SVMSMOTE & 0.9692 & 0.9836 & \textbf{0.9989} & & 0.9378 & 0.9656 & 0.9936 \\
 & \textbf{KWSMOTE} & \textbf{0.9749} & \textbf{0.9859} & 0.9984 & & \textbf{0.9424} & 0.9670 & \textbf{0.9931} \\
\midrule
Diabetes & raw & 0.7298 & 0.7660 & 0.8040 & & 0.6861 & 0.7279 & 0.7722 \\
 & SMOTE & 0.7310 & 0.7635 & 0.7976 & & 0.6646 & 0.7140 & 0.7672 \\
 & SNOCC & 0.7280 & 0.7650 & \textbf{0.8060} & & 0.6740 & 0.7200 & 0.7700 \\
 & BorderlineSMOTE & 0.7198 & 0.7506 & 0.7829 & & 0.6547 & 0.7093 & 0.7689 \\
 & SVMSMOTE & 0.7346 & 0.7447 & 0.8023 & & 0.6704 & 0.7186 & 0.7694 \\
 & \textbf{KWSMOTE} & \textbf{0.7310} & \textbf{0.7660} & 0.8040 & & \textbf{0.6866} & \textbf{0.7286} & \textbf{0.7732} \\
\bottomrule
\end{tabular*}
\label{tab:results}
\end{table*}

The core methodology would deal with the performance of two classifiers: Support Vector Machine and Random Forest, in combination with three data sampling strategies: no sampling (raw data), the classic SMOTE method, SNOCC, our proposed method and two SMOTE variants inculde BorderlineSMOTE which focuses on generating synthetic samples near the decision boundary; SVMSMOTE which uses Support Vector Machines (SVM) to guide the synthetic sample generation process. The choice of SVM, therefore, is based on its high performance in high-dimensional spaces and its ability to build decision boundaries that best give the margin between classes; it results in a generalized classification task most importantly for the complex non-linear relationship. On the other hand, the Random Forest approach was chosen to perform well with large datasets, especially those having higher dimensions, and because of its robustness against overfitting due to the ensemble method. We applied both SVM and Random Forest to cover a great variety of classification situations and get the best from each model.

We don't need extra procedures to deal with outliers as our method can effectively avoid their influence.
Finally, the datasets were divided into training and testing sets at a 7:3 ratio using stratified random sampling to maintain the distribution of classes.

We benchmarked using three main evaluation metrics, F1-score, G-mean, and the Area Under the Receiver Operating Characteristic Curve (AUC).

\begin{itemize}
    \item \textbf{F1-Score:} This metric balances precision and recall, providing a single score that considers both false positives and false negatives. It is beneficial for imbalanced datasets, where focusing solely on accuracy might be misleading. The F1-Score is calculated as:
    \[
    \text{F1-Score} = 2 \times \frac{\text{Precision} \times \text{Recall}}{\text{Precision} + \text{Recall}},
    \]
    where Precision is given by:
    \[
    \text{Precision} = \frac{TP}{TP + FP},
    \]
    and Recall is:
    \[
    \text{Recall} = \frac{TP}{TP + FN}.
    \]
     True Positives (TP) represents the number of positive class samples that are correctly identified by the classifier. False Positives (FP) and False Negatives (FN) represents the number of negative and positive class samples that are incorrectly classified as positive and negative, respectively.
    
    \item \textbf{G-mean:} The G-mean represents the geometric mean of True Positive Rate (TPR) and True Negative Rate (TNR). It balances the trade-off between the detection of the positive class and the correct identification of the negative class, offering a robust measure of performance in the presence of class imbalance. It is calculated as:
    \[
    \text{G-mean} = \sqrt{\text{TPR} \times \text{TNR}},
    \]
    where True Positive Rate is:
    \[
    \text{TPR} = \frac{TP}{TP + FN},
    \]
    and True Negative Rate is:
    \[
    \text{TNR} = \frac{TN}{TN + FP}.
    \]
    
    \item \textbf{AUC:} The Area Under the Receiver Operating Characteristic Curve (AUC) reflects the classifier's ability to differentiate between classes, with higher values indicating better performance. The AUC provides a summary of the trade-off between the true positive rate and the false positive rate across different threshold settings, where a score of 1.0 indicates perfect classification, and 0.5 represents performance equivalent to random guessing.
\end{itemize}

\subsection{Results and Analysis}

The experimental results for the Blood Transfusion, Haberman, Breast Cancer Wisconsin (Diagnostic), and Diabetes datasets are summarized in Table \ref{tab:results}. Bolded cells indicate the best performance metrics achieved by different sampling strategies.

The results demonstrate that the KWSMOTE method consistently improves classification performance compared to the raw and classic SMOTE data, particularly in terms of F1-score, AUC, and G-mean.

For the Blood Transfusion dataset, KWSMOTE achieved the best results for most metrics, especially with the Random Forest classifier, where it significantly outperformed other methods. Specifically, KWSMOTE improved the G-mean by 5.65\% for the Random Forest classifier. In addition, KWSMOTE is still the best in terms of AUC and G-mean for the SVM classifier. Notably, SVMSMOTE also showed competitive performance, achieving an AUC of 0.6734 for SVM and 0.6451 for Random Forest, making it one of the stronger resampling techniques. BorderlineSMOTE performed slightly better than SMOTE in most cases, with an AUC of 0.6813 for SVM, the highest among all methods.

In the Haberman dataset, KWSMOTE continued to show superior results across all metrics, outperforming both SMOTE and the raw data. KWSMOTE increased the F1-score by 3.03\% for the SVM classifier and by 3.04\% for the Random Forest classifier compared to SMOTE. Additionally, the AUC values improved by 6.30\% for Random Forest and 4.45\% for SVM. BorderlineSMOTE and SVMSMOTE also provided notable improvements over SMOTE, with SVMSMOTE achieving an AUC of 0.6823 for SVM and 0.6503 for Random Forest, reinforcing its potential in imbalanced data classification.

KWSMOTE demonstrated consistent superiority for the Breast Cancer Wisconsin (Diagnostic) dataset, particularly with the Random Forest classifier, achieving a G-mean of 0.9859. SVMSMOTE achieved the highest AUC of 0.9989 for Random Forest, showing that it can be a competitive alternative in certain cases. The Diabetes dataset also highlighted KWSMOTE's strengths. The G-mean for KWSMOTE with the SVM classifier increased by 2.02\% compared to SMOTE, and by 0.33\% for the Random Forest classifier compared to SNOCC, which performed best in AUC with an improvement of 1.05\% over SMOTE. Meanwhile, SVMSMOTE attained a G-mean of 0.7447 for Random Forest and an AUC of 0.7694 for SVM, further emphasizing its utility. BorderlineSMOTE, while slightly behind, still improved upon SMOTE, achieving an AUC of 0.7689 for SVM.

Overall, KWSMOTE consistently outperformed the raw, SMOTE, BorderlineSMOTE, SVMSMOTE, and SNOCC methods across all datasets and classifiers, providing significant improvements in classification metrics. The findings underscore KWSMOTE's effectiveness as a robust sampling strategy for imbalanced datasets, offering notable gains in F1-score, AUC, and G-mean. Meanwhile, SVMSMOTE and BorderlineSMOTE also demonstrated their viability as alternative techniques, particularly in datasets where AUC improvements are critical. These insights highlight the importance of selecting appropriate resampling strategies tailored to the dataset and classification model in use.

\subsection{Discussion: Why AUC is Less Informative}

While AUC is a widely used metric for evaluating classifier performance, it can be less informative in certain practical scenarios. For instance, consider a dataset where the positive class is extremely rare. In such a situation, a classifier may achieve a high AUC by simply being very good at ranking instances, even if its absolute performance (e.g., precision or recall) in the positive class is poor. 

As an example, imagine a medical diagnostic test for a rare disease, analogous to the class imbalance observed in datasets such as Blood Transfusion or Haberman. Suppose the classifier correctly ranks almost all patients with the disease above those without it, resulting in an impressive AUC. However, if the threshold for a positive prediction is not carefully chosen, the test might still yield a large number of false positives, causing unnecessary anxiety and additional testing. Similarly, if a classifier is evaluated solely based on AUC without considering the G-mean, it may appear effective despite failing to properly distinguish the minority class, leading to suboptimal real-world performance.

\section{Conclusion}
In this paper, we introduced KWSMOTE, a novel enhancement to the SMOTE algorithm, specifically designed to address the challenges of class imbalance in datasets. By incorporating convex combination and kernel-based weighting function into the oversampling process, KWSMOTE generates synthetic samples that better represent the true distribution of the minority class. This method overcomes some of the limitations of traditional SMOTE, such as boundary marginalization and the generation of less realistic samples.

Through extensive experimentation on four real-world datasets, including Blood Transfusion, Haberman, Breast Cancer Wisconsin (Diagnostic), and Diabetes, we demonstrated that KWSMOTE significantly improves classification performance compared to both the raw data, original SMOTE, and SNOCC. Across multiple evaluation metrics, including F1-score, G-mean, and AUC, our approach consistently outperformed existing methods, particularly in complex classification scenarios involving imbalanced data.

The results show that KWSMOTE is a robust and effective tool for imbalanced dataset classification, offering better accuracy and model reliability. Future work may explore integrating KWSMOTE with other advanced machine learning techniques, such as deep learning models, to further enhance its applicability in large-scale and high-dimensional datasets.


\bibliographystyle{IEEEtran}
\bibliography{ref}

\end{document}